\documentclass[11pt,letterpaper]{article}

\usepackage{bm, mathrsfs, graphics,float,amssymb,amsmath,subeqnarray,graphicx,epstopdf,subfigure, enumerate, color,natbib,amsthm,hyperref,array}
\usepackage[utf8]{inputenc}
\usepackage[margin=1in]{geometry}
\usepackage{setspace}
\usepackage{algorithm}
\usepackage{algorithmic}
\hypersetup{
  pdfauthor={Haishan Ye},
  pdftitle={Logarithmic High-Probability Regret for Strongly Convex Online Convex Optimization with Two-Point Bandit Feedback}
}

\newcommand{\EB}{\mathbb{E}}
\newcommand{\SB}{\mathbb{S}}
\newcommand{\bx}{\bm{x}}
\newcommand{\by}{\bm{y}}
\newcommand{\bu}{\bm{u}}
\newcommand{\bg}{\bm{g}}
\newcommand{\bv}{\bm{v}}

\newcommand{\Cabs}{C_{\rm abs}}

\newcommand{\cK}{\mathcal{K}}
\newcommand{\cF}{\mathcal{F}}
\newcommand{\cB}{\mathcal{B}}

\newcommand{\norm}[1]{\left\|#1\right\|}

\newtheorem{lemma}{Lemma}
\newtheorem{theorem}{Theorem}

\newtheorem{definition}{Definition}
\newtheorem{corollary}[theorem]{Corollary}

\ifx\assumption\undefined
\newtheorem{assumption}{Assumption}
\fi

\begin{document}

\begin{titlepage}
\centering
\vspace*{0.7in}
{\Large\bfseries Logarithmic High-Probability Regret for Strongly Convex Online Convex Optimization with Two-Point Bandit Feedback\par}
\vspace{1.2em}
{\large Haishan Ye\par}
\vspace{0.35em}
{\normalsize School of Management\\
Xi'an Jiaotong University\\
\texttt{hsye\_cs@outlook.com}\par}
\vspace{2em}
\begin{abstract}
	We study online convex optimization (OCO) with two-point bandit feedback against a non-anticipating adaptive adversary.
	In this setting, a learner competes with an adversarial sequence of convex losses while observing each loss only through two function evaluations.
	For strongly convex losses, Agarwal, Dekel, and Xiao~\citeyearpar{agarwal2010optimal} proved a comparator-wise logarithmic regret bound in expectation. Consequently, by minimizing outside the probability space, their result yields a pseudo-regret guarantee of the form $\EB A_T-\min_{x\in\mathcal K}\EB L_T(x)$, where $A_T$ is the algorithm's two-query cumulative loss and $L_T(x)$ is the comparator's cumulative loss.
	They asked whether a logarithmic high-probability guarantee is achievable in the same two-point strongly convex setting.
	Our main theorem provides the corresponding fixed-comparator high-probability statement: for any comparator $x\in\mathcal K$ fixed independently of the algorithmic random directions, the standard two-point projected gradient method guarantees, with probability at least $1-\delta$, a two-query regret bound of order
	\[
	O\left(\frac{dG^2}{\mu}\left(\log T+\log(1/\delta)\right)+dGD\log(1/\delta)+G\log T\left(1+\frac{D}{r}\right)\right).
	\]
	At the comparator-wise level, our leading horizon-dependent term is linear in $d$, compared with the $d^2$-type term in the original analysis of Agarwal, Dekel, and Xiao.
	The key ingredient is a high-confidence analysis that simultaneously absorbs the martingale error into strong convexity and preserves the linear-in-dimension estimator control of the two-point method.
	A deterministic covering argument then yields a realized full-comparator guarantee against $\min_{x\in\mathcal K}L_T(x)$, preserving logarithmic dependence on $T$ at the cost of the standard covering-number factor.
	
\end{abstract}
\vfill
\end{titlepage}
\setcounter{page}{2}

\section{Introduction}

Online convex optimization (OCO) is a repeated decision problem between a learner and an \textit{adversary}.
At each round $t = 1, \dots, T$, the learner chooses an action $x_t$ from a fixed convex set $\mathcal{K} \subset \mathbb{R}^d$. A non-anticipating adaptive adversary then selects a convex loss function $\ell_t: \mathcal{K} \to \mathbb{R}$ using the past history and, in the two-point protocol studied here, before the learner's fresh random direction at round $t$ is sampled.
While full-information OCO is well understood, the \textit{bandit feedback} setting, where only function values of $\ell_t$ are observed, remains substantially more challenging because gradients are unavailable.
Two-point feedback is the minimal zeroth-order interface that permits an unbiased gradient estimate of a smoothed loss, making it a natural test case for whether bandit feedback can match full-information fast rates.
Strong convexity is the canonical assumption under which full-information OCO improves from $\sqrt T$-type regret to logarithmic regret.
High-probability guarantees are also important in online decision problems: an expected logarithmic bound may still allow rare but large deviations, whereas a high-confidence bound controls almost every realization of the learner's randomization.
The main difficulty is that the two-point gradient estimator is noisy, and this noise is coupled with the learner's trajectory under an adaptive adversary.
Consequently, the usual full-information logarithmic-regret proof and a direct Azuma-type concentration argument do not by themselves yield the desired rate.

For compactness, write
\begin{equation*}
	A_T:=
	\sum_{t=1}^T \frac{\ell_t(x_t+\alpha u_t)+\ell_t(x_t-\alpha u_t)}{2},
	\qquad
	L_T(x):=\sum_{t=1}^T\ell_t(x).
\end{equation*}
For a fixed comparator $x\in\mathcal K$, the two-point protocol analyzed in this paper measures performance by $R_T^{\rm 2pt}(x)=A_T-L_T(x)$.
We distinguish three comparator semantics. A comparator-wise expected bound has the form
\[
\forall x\in\mathcal K,\qquad \EB[A_T-L_T(x)]\le B .
\]
Since $x_{\rm exp}^\star\in\arg\min_x \EB L_T(x)$ is deterministic, one may minimize outside the probability space and obtain the pseudo-regret bound
\[
\EB A_T-\min_{x\in\mathcal K}\EB L_T(x)\le B .
\]
This is different from the realized hindsight regret
\[
A_T-\min_{x\in\mathcal K}L_T(x),
\]
whose minimizer may depend on the realized loss sequence and hence on the learner's random trajectory under an adaptive adversary. A fixed-comparator high-probability statement has yet another logical form: for each fixed $x$, it provides a high-probability event that may depend on $x$. It does not automatically imply a single event on which the bound holds for all $x\in\mathcal K$. Such a uniform event is precisely what is needed to compete with the realized hindsight minimizer. Accordingly, our main theorem is stated for a fixed comparator, and the realized full-comparator guarantee is derived separately through a deterministic covering argument.

To bridge the gap between bandit and full-information feedback, Agarwal, Dekel, and Xiao~\citeyearpar{agarwal2010optimal} (ADX) introduced the multi-point bandit setting. They showed that two randomized function evaluations suffice to construct gradient estimators for smoothed losses. Their strongly convex logarithmic theorem is comparator-wise in expectation: for every deterministic $x\in\mathcal K$, it controls $\EB[A_T-L_T(x)]$, with an $O(d^2\log T)$ horizon-dependent term up to problem parameters. Since the minimizer of $x\mapsto \EB L_T(x)$ is deterministic, this immediately yields their pseudo-regret corollary
\[
\EB A_T-\min_{x\in\mathcal K}\EB L_T(x).
\]
However, this is not the same as controlling the realized hindsight regret $A_T-\min_x L_T(x)$, nor does it provide a high-probability event uniform over all comparators.
Their high-probability theorem is pointwise in a fixed comparator and does not achieve logarithmic dependence on $T$. They also explain that competing with the realized hindsight minimizer under adaptive losses requires first obtaining a fixed-comparator high-probability event and then uniformizing it over a cover or a barycentric spanner. The missing ingredient is therefore a logarithmic fixed-comparator high-probability theorem in the two-point strongly convex setting. The difficulty is not unbounded tails, but rather sharp high-probability control of the large-norm, trajectory-dependent estimator noise without reverting to the crude deterministic bound $\|\bg_t\|\le dG$.

We prove this fixed-comparator high-probability theorem for two-point feedback OCO with $\mu$-strongly convex losses. Our main result establishes that for any comparator $x\in\mathcal K$ fixed independently of the algorithmic random directions, with probability at least $1-\delta$,
\begin{equation*}
	R_T^{\rm 2pt}(x) \leq O\left(\frac{dG^2}{\mu}\left(\log T+\log (1/\delta)\right)+dGD\log(1/\delta)+G\log T\left(1+\frac{D}{r}\right)\right).
\end{equation*}
Suppressing the displayed domain-dependent terms gives the shorthand
\begin{equation*}
	R_T^{\rm 2pt}(x) \leq O\left( \frac{dG^2(\log T + \log (1/\delta))}{\mu} \right).
\end{equation*}
At the same comparator-wise level, the expected strongly convex analysis of Agarwal, Dekel, and Xiao has an $O(d^2\log T)$ horizon-dependent term up to problem parameters. Our fixed-comparator high-probability theorem has leading horizon-dependent term $O(d\log T)$. Thus the dimension comparison is made at the comparator-wise level; the realized full-comparator statement in this paper is obtained only after an additional covering step. The covering corollary controls a stronger comparator semantics than the expected pseudo-regret corollary: it compares to the realized hindsight minimizer on the same sample path. The bound is numerically larger because of the cover, but it applies to the realized hindsight minimizer on that sample path.

\paragraph{Our contributions.}
We summarize the main contributions as follows.
\begin{enumerate}
\item We prove a fixed-comparator high-probability regret bound for two-point bandit OCO with strongly convex losses under a non-anticipating adaptive adversary. For any deterministic comparator $x\in\mathcal K$, the leading horizon-dependent term is $O((dG^2/\mu)\log T)$. This gives the logarithmic high-probability analogue of the comparator-wise expected guarantee of Agarwal, Dekel, and Xiao~\citeyearpar{agarwal2010optimal}, while improving the fixed-comparator dimension dependence of the horizon-dependent term from $d^2$ to $d$.
\item We derive a realized full-comparator high-probability regret bound against
\[
\min_{x\in\mathcal K}\sum_{t=1}^T\ell_t(x)
\]
by applying the fixed-comparator theorem uniformly over a deterministic cover. This controls the sample-path hindsight comparator, which may be random under adaptive losses. The price is the standard covering-number factor.
\item At the technical level, the proof isolates martingale noise and estimator energy, the two obstacles that block a direct high-probability analogue of the expected analysis. A Freedman-style exponential supermartingale absorbs the martingale term into strong convexity, while spherical concentration controls the weighted estimator energy at the $d\log T$ scale.
\end{enumerate}

\paragraph{Proof overview.}
The proof starts from the logarithmic-regret OGD decomposition for strongly convex losses. In the two-point bandit setting, two obstacles prevent a direct high-probability analogue of this argument.
The first obstacle is the martingale error
\[
\sum_{t=1}^T\langle \nabla \hat{\ell}_t(\bx_t)-\bg_t,\bx_t-\bx\rangle
\]
which appears because the estimated gradient replaces the true gradient of the smoothed loss.
A direct Hoeffding--Azuma bound for this term loses a factor of order $\sqrt T$.
We instead use a Freedman-style exponential supermartingale whose predictable variance is proportional to $\sum_t\|\bx_t-\bx\|^2$; the resulting contribution is absorbed by the strong-convexity curvature term in the OGD decomposition.
The second obstacle is the weighted estimator-size term $\sum_{t=1}^T\|\bg_t\|^2/t$: a deterministic norm bound would lose the desired linear dependence on the dimension.
For this term, we refine the geometric argument of Shamir~\citeyearpar{shamir2017optimal}: after conditioning on the past, each queried loss is a Lipschitz function on the sphere, so spherical concentration yields weighted sub-exponential control at scale $(G\alpha)^2/d$.
This cancels one factor of $d$ from the estimator norm and yields the desired $d\log T$ behavior.

\paragraph{Related work.}
\noindent Flaxman, Kalai, and McMahan~\citeyearpar{Flaxman2005Online} initiated the study of bandit convex optimization via gradient descent without gradients, using one function evaluation per round to estimate gradients of smoothed losses.
The multi-point bandit model and the conversion from fixed-comparator high-probability bounds to uniform comparator bounds through finite covers are due to Agarwal, Dekel, and Xiao~\citeyearpar{agarwal2010optimal}.
Their two-point method obtains comparator-wise logarithmic regret for strongly convex losses in expectation, hence pseudo-regret after minimizing outside the probability space, but their high-probability guarantee does not achieve logarithmic dependence on $T$.
Recent work of Yu, Yan, and Zhao~\citeyearpar{yu2026improved} shows that, in gradient-variation two-point bandit convex optimization, the strongly convex bound can have linear dependence on $d$, improving the $d^2$-type dimension dependence inherited from the ADX two-point analysis to a $d$-dependent bound in that variation-adaptive setting.
Logarithmic regret for strongly convex full-information OCO follows the classical OGD analysis of Hazan, Agarwal, and Kale~\citeyearpar{hazan2007logarithmic}. Broader bandit and zeroth-order complexity phenomena are discussed by Abernethy, Agarwal, Bartlett, and Rakhlin~\citeyearpar{abernethy2009stochastic}, Bubeck and Cesa-Bianchi~\citeyearpar{Bubeck2012Regret}, and Duchi, Jordan, Wainwright, and Wibisono~\citeyearpar{Duchi2015Optimal}.

The linear-in-dimension estimator control used in expectation by Shamir~\citeyearpar{shamir2017optimal} is a main technical precursor to our work.
The high-confidence part of the present paper combines a curvature-absorbing martingale argument with concentration of measure on the sphere, using tools in the spirit of Ledoux~\citeyearpar{ledoux2001concentration} and Vershynin~\citeyearpar{Vershynin2026}.

\section{Notation and Preliminaries}

\subsection{Notation}
We use the following notation throughout the analysis.
Let $\| \cdot \|$ denote the Euclidean norm in $\mathbb{R}^d$, and let $\mathcal{B} = \{\bx \in \mathbb{R}^d : \|\bx\| \le 1\}$ be the unit ball centered at the origin.
We assume that the convex set $\mathcal{K}$ is compact, has nonempty interior in its ambient space, and contains the origin after translation. More specifically, there exist $r, D > 0$ such that
\begin{equation}\label{eq:DB}
r\mathcal{B} \subseteq\mathcal{K} \subseteq D\mathcal{B}.
\end{equation}

\begin{lemma}[Feasibility of queried points]\label{lem:query_feasible}
Suppose Eq.~\eqref{eq:DB} holds. If $0<\xi\le 1$ and $0<\alpha\le \xi r$, then for every $\bx\in (1-\xi)\cK$, every $\bu\in\SB^{d-1}$, and every $\bv\in\cB$,
\[
\bx+\alpha\bv\in\cK,\qquad \bx+\alpha\bu\in\cK,\qquad \bx-\alpha\bu\in\cK.
\]
If, in addition, $\alpha\le \xi r/2$, then these points lie in $\operatorname{int}(\cK)$.
\end{lemma}
\begin{proof}
Write $\bx=(1-\xi)\by$ for some $\by\in\cK$. Since $\alpha\le \xi r$, for any $\bv\in\cB$ we have $(\alpha/\xi)\bv\in r\cB\subseteq\cK$. Hence
\[
\bx+\alpha\bv=(1-\xi)\by+\xi\left(\frac{\alpha}{\xi}\bv\right)\in\cK
\]
by convexity. The two spherical query points follow by taking $\bv=\bu$ and $\bv=-\bu$.
If $\alpha\le \xi r/2$, then $(\alpha/\xi)\bv\in (r/2)\cB\subset \operatorname{int}(\cK)$. A nontrivial convex combination of a point in $\cK$ and a point in $\operatorname{int}(\cK)$ lies in $\operatorname{int}(\cK)$, giving the final claim.
\end{proof}

Throughout the paper, the loss functions $\ell_t(\cdot)$ are
$G$-Lipschitz continuous and $\mu$-strongly convex.
\begin{assumption}
Each loss function $\ell_t$ is $G$-Lipschitz continuous: there exists a constant $G\geq 0$ such that
\begin{equation*}
	|\ell_t(\bx) - \ell_t(\by)| \le G \|\bx - \by\|, \quad \forall \bx, \by \in \mathcal{K}, \forall t.
\end{equation*}
\end{assumption}
\begin{assumption}
For $\mu \ge 0$, the function $\ell$ is called $\mu$-strongly convex on the set $\mathcal{K}$ if
\begin{equation*}
	\ell(\bx) \ge \ell(\by) + \langle s_{\by}, \bx - \by\rangle + \frac{\mu}{2} \|\bx - \by\|^2, \quad \forall \bx, \by \in \mathcal{K},\quad \forall s_{\by}\in\partial \ell(\by).
\end{equation*}
\end{assumption}
\begin{assumption}[Non-anticipating adversary]\label{ass:nonanticipating}
For each round $t$, let $\cF_{t-1}$ denote the sigma-field containing all learner and adversary randomness and all observations available immediately before sampling $\bu_t$; in particular, $\bx_t$ has been computed and the adversary has already selected $\ell_t$. We assume that $\bx_t$ and $\ell_t$ are $\cF_{t-1}$-measurable, and that $\bu_t$ is independent of $\cF_{t-1}$ and uniformly distributed on $\SB^{d-1}$. Let $\cF_t$ be the sigma-field after $\bu_t$ and the two queried values at round $t$ are observed.
\end{assumption}

When applying Orlicz-norm or moment-generating-function bounds conditionally on $\cF_{t-1}$, we mean the corresponding unconditional statement after freezing the history. All constants are uniform over the frozen history.

\subsection{Preliminaries}

For a smoothing radius $\alpha>0$, define the smoothed loss
\begin{equation}\label{eq:ell_h}
\hat{\ell}_t(\bx) = \EB_{\bv} \left[ \ell_t(\bx + \alpha \bv) \right],
\end{equation}
where $\bv$ is a uniform random vector in the unit ball $\cB$.
Throughout the analysis, $\hat{\ell}_t$ is evaluated only at points in $(1-\xi)\cK$. When $\alpha\le \xi r$, Lemma~\ref{lem:query_feasible} ensures that $\bx+\alpha\bv\in\cK$, so the smoothing operation and the two queried losses are well-defined without assuming any extension of $\ell_t$ outside $\cK$. In the main theorem we use the stronger choice $\alpha=\xi r/2$, which gives a positive interior margin for the smoothing ball and the two query points.
The smoothed loss preserves the following properties of the original loss.
\begin{lemma}[Nesterov and Spokoiny~\citeyearpar{Nesterov2017}]
	\label{lem:mu}
	If $\ell(\bx)$ is $G$-Lipschitz continuous, then $\hat{\ell}(\bx)$ defined in Eq.~\eqref{eq:ell_h} is also $G$-Lipschitz continuous.
	If $\ell(\bx)$ is $\mu$-strongly convex, then $\hat{\ell}(\bx)$ defined in Eq.~\eqref{eq:ell_h} is also $\mu$-strongly convex.
\end{lemma}

We use the following unbiased estimator of $\nabla \hat{\ell}_t(\bx_t)$:
\begin{equation}\label{eq:g}
	\bg_t = d \cdot \frac{\ell_t(\bx_t + \alpha \bu_t) - \ell_t(\bx_t - \alpha \bu_t)}{2\alpha}\cdot \bu_t, \mbox{ where } \bu_t \sim \mathbb{S}^{d-1},
\end{equation}
that is, $\bu_t$ is uniformly drawn from the unit sphere.
\begin{lemma}[Lemma 10 of Shamir~\citeyearpar{shamir2017optimal}]
	\label{lem:g_prop}
	Under Assumption~\ref{ass:nonanticipating}, the approximate gradient $\bg_t$ defined in Eq.~\eqref{eq:g} satisfies
	\[
	\EB[\bg_t \mid \cF_{t-1}] = \nabla \hat{\ell}_t(\bx_t),
	\qquad
	\EB[\|\bg_t\|^2 \mid \cF_{t-1}] \leq c_0 d G^2,
	\]
	where $c_0$ is an absolute constant.
\end{lemma}
The lemma shows that, conditionally on the complete pre-query history, $\bg_t$ is an unbiased estimate of
$\nabla \hat{\ell}_t(\cdot)$ and has second moment linear in the dimension $d$.
Algorithm~\ref{alg:sgd} uses this estimator in a projected online gradient step.
This is the standard two-point projected gradient method of Agarwal, Dekel, and Xiao~\citeyearpar{agarwal2010optimal}, with parameter choices later specialized to the high-probability analysis.

\begin{lemma}[Lemma 2 of Agarwal, Dekel, and Xiao~\citeyearpar{agarwal2010optimal}]
\label{lem:diff}
Assume $0<\xi\le 1$, $\alpha\le \xi r$, and $\bx_t\in(1-\xi)\cK$ for all $t$.
For any $\bx \in \cK$, it holds that
\begin{equation}\label{eq:diff}
\begin{aligned}
&\sum_{t=1}^{T}\frac{ \ell_t(\bx_t + \alpha \bu_t) + \ell_t(\bx_t - \alpha \bu_t)}{2} - \sum_{t=1}^T \ell_t(\bx)
\\
\leq&
\sum_{t=1}^T \hat{\ell}_t(\bx_t) - \sum_{t=1}^T \hat{\ell}_t((1-\xi)\bx) + 3TG\alpha + TGD\xi.
\end{aligned}
\end{equation}
\end{lemma}

\begin{algorithm}[tb]
\caption{Online gradient descent with two queries per round}
	\label{alg:sgd}
	\begin{algorithmic}[1]
\STATE \textbf{Input:} Learning rates $\eta_t$, exploration parameter $\alpha$, and shrinkage coefficient $\xi$.
		\STATE Set $\bx_1 = 0$
		\FOR{$t = 1, \dots, T$}
		\STATE The adversary selects $\ell_t$, measurable with respect to $\cF_{t-1}$, after seeing $\bx_t$ and before $\bu_t$ is sampled.
		\STATE Pick a unit vector $\bu_t$ uniformly at random from the unit sphere.
		\STATE Observe $\ell_t(\bx_t + \alpha \bu_t)$ and $\ell_t(\bx_t - \alpha \bu_t)$.
\STATE Construct the approximate gradient $\bg_t$ as in Eq.~\eqref{eq:g}.
		\STATE Update $\bx_{t+1} = \Pi_{(1-\xi)\mathcal{K}}(\bx_t - \eta_t \bg_t)$.
		\ENDFOR
	\end{algorithmic}
\end{algorithm}

\section{Main Results}

We now state the main guarantees. Theorem~\ref{thm:main} gives the
high-probability bound for a fixed comparator. Corollary~\ref{cor:uniform_cover}
then yields a realized full-comparator regret guarantee
through a deterministic covering argument.
\begin{theorem}
	\label{thm:main}
	Suppose the loss function $\ell_t(\cdot)$ is $\mu$-strongly convex with $\mu>0$ and $G$-Lipschitz continuous.
	The constraint set $\cK$ is convex and satisfies Eq.~\eqref{eq:DB}.
	Let the sequence $\{\bx_t\}$ be generated by Algorithm~\ref{alg:sgd}. 
	Assume that the adversary satisfies Assumption~\ref{ass:nonanticipating}. Given $0<\delta<1$ and $T\ge 2$ such that $2\log(eT)/T\le r$, set the step size $\eta_t = \frac{2}{\mu t}$, $\alpha = \frac{\log(eT)}{T}$ and $\xi = 2\alpha / r$.
	Fix any comparator $\bx\in\cK$ independently of the algorithmic random directions. Then, with probability at least $1-\delta$,
	\begin{equation}\label{eq:main}
		\begin{aligned}
			&\sum_{t=1}^{T}\frac{ \ell_t(\bx_t + \alpha \bu_t) + \ell_t(\bx_t - \alpha \bu_t)}{2} - \sum_{t=1}^T \ell_t(\bx)
			\\
			\leq&
			\Cabs \cdot \frac{d G^2}{\mu}\cdot \left( \log(eT) + \log \frac{e}{\delta} \right) 
			+
	4(d+1)GD\log\frac{e}{\delta}
			+
			G\log(eT)\left(3 + \frac{2D}{r}\right),
		\end{aligned}
	\end{equation}
	where $\Cabs$ is an absolute constant.
\end{theorem}

The probability statement in Theorem~\ref{thm:main} is pointwise in the comparator:
for each comparator fixed independently of the random directions, the high-probability event may depend on that comparator. Corollary~\ref{cor:uniform_cover} provides the corresponding uniform statement over a finite cover.

With the parameter choice in Theorem~\ref{thm:main}, Algorithm~\ref{alg:sgd} satisfies the following fixed-comparator two-query regret bound: for each comparator $\bx$ fixed independently of the random directions, with probability at least $1-\delta$,
\begin{equation}\label{eq:reg_all}
	R_T^{\rm 2pt}(\bx) \leq O\left(\frac{dG^2}{\mu}\left(\log T+\log(1/\delta)\right)+dGD\log(1/\delta)+G\log T\left(1+\frac{D}{r}\right)\right).
\end{equation}
Equation~\eqref{eq:reg_all} is the high-probability analogue, at the comparator-wise level, of the expected guarantee of Agarwal, Dekel, and Xiao~\citeyearpar{agarwal2010optimal}. Their expected bound has an $O(d^2\log T)$ horizon-dependent term, whereas the leading term in Eq.~\eqref{eq:reg_all} is linear in $d$. In this precise sense, Theorem~\ref{thm:main} resolves the fixed-comparator high-probability logarithmic-regret form of the question left open by Agarwal, Dekel, and Xiao~\citeyearpar{agarwal2010optimal}.
The leading fixed-comparator horizon-dependent term is logarithmic in $T$ and linear in $d$, up to constants and the displayed domain-dependent confidence terms.

\begin{center}
\small
\setlength{\tabcolsep}{3pt}
\begin{tabular}{>{\raggedright\arraybackslash}p{0.25\linewidth}>{\raggedright\arraybackslash}p{0.25\linewidth}>{\raggedright\arraybackslash}p{0.18\linewidth}>{\raggedright\arraybackslash}p{0.22\linewidth}}
\hline
Result & Comparator semantics & Guarantee type & Horizon-dependent term \\
\hline
ADX expected theorem~\citeyearpar{agarwal2010optimal} & fixed deterministic comparator $x$ & expectation & $O(d^2\log T)$ \\
ADX expected corollary & pseudo-regret $\EB A_T-\min_x\EB L_T(x)$ & expectation & $O(d^2\log T)$ \\
ADX high-probability theorem & fixed deterministic comparator $x$ & high probability & not logarithmic in $T$ \\
Theorem~\ref{thm:main} & fixed deterministic comparator $x$ & high probability & $O(d\log T)$ leading term \\
Corollary~\ref{cor:uniform_cover} & realized hindsight comparator $\min_x L_T(x)$ via cover & high probability & logarithmic in $T$; $d\log T$ cover overhead \\
\hline
\end{tabular}
\end{center}

The first two ADX rows are expectation statements. The second row is a pseudo-regret bound obtained by minimizing outside the probability space. The last row is a realized sample-path comparator guarantee and therefore requires a uniform high-probability event over the comparator set.
The table separates the fixed-comparator improvement from the uniformization step. The linear-in-$d$ leading term is a fixed-comparator statement; the realized full-comparator corollary preserves logarithmic dependence on $T$ but pays the standard covering-number factor.

\begin{corollary}[Realized full-comparator regret via covering]\label{cor:uniform_cover}
Suppose the assumptions of Theorem~\ref{thm:main} hold. Let $0<\varepsilon\le D$, and let
$\mathcal N_\varepsilon\subset\cK$ be an $\varepsilon$-net of $\cK$ in Euclidean norm with
$|\mathcal N_\varepsilon|\le (3D/\varepsilon)^d$. Then with probability at least $1-\delta$,
\begin{align}
&\sum_{t=1}^{T}\frac{ \ell_t(\bx_t + \alpha \bu_t) + \ell_t(\bx_t - \alpha \bu_t)}{2}
-\min_{\bx\in\cK}\sum_{t=1}^T\ell_t(\bx)
\nonumber\\
\leq\;&
\Cabs \cdot \frac{d G^2}{\mu}
\left(\log(eT)+\log\frac{e}{\delta}+d\log\frac{3D}{\varepsilon}\right)
\nonumber\\
&+
4(d+1)GD
\left(\log\frac{e}{\delta}+d\log\frac{3D}{\varepsilon}\right)
+
G\log(eT)\left(3+\frac{2D}{r}\right)
+
GT\varepsilon .
\label{eq:uniform_cover}
\end{align}
In particular, taking $\varepsilon=D/T$ and using $T\ge 2$ gives
\begin{align}
&\sum_{t=1}^{T}\frac{ \ell_t(\bx_t + \alpha \bu_t) + \ell_t(\bx_t - \alpha \bu_t)}{2}
-\min_{\bx\in\cK}\sum_{t=1}^T\ell_t(\bx)
\nonumber\\
\leq\;&
\Cabs \cdot \frac{d G^2}{\mu}
\left(\log(eT)+\log\frac{e}{\delta}+d\log(eT)\right)
\nonumber\\
&+
4(d+1)GD
\left(\log\frac{e}{\delta}+d\log(3T)\right)
+
G\log(eT)\left(3+\frac{2D}{r}\right)
+ GD .
\label{eq:uniform_cover_simplified}
\end{align}
Thus the realized full-comparator regret remains logarithmic in the horizon. With this
choice of $\varepsilon$, the cover contributes the displayed terms of order
$O((d^2G^2/\mu)\log T+d^2GD\log T)$, reflecting the additional factor
$\log|\mathcal N_\varepsilon|=O(d\log T)$.
\end{corollary}

\begin{proof}
The existence of an $\varepsilon$-net with the stated cardinality follows from
the inclusion $\cK\subseteq D\cB$ and the standard Euclidean covering number
bound for a $d$-dimensional ball. Fix such a net
$\mathcal N_\varepsilon$. Applying
Theorem~\ref{thm:main} to every $\by\in\mathcal N_\varepsilon$ with failure
probability $\delta/|\mathcal N_\varepsilon|$ and taking a union bound, we obtain
that, with probability at least $1-\delta$, Eq.~\eqref{eq:main} holds
simultaneously for all $\by\in\mathcal N_\varepsilon$, with
$\log(e/\delta)$ replaced by $\log(e|\mathcal N_\varepsilon|/\delta)$.
Since $|\mathcal N_\varepsilon|\le (3D/\varepsilon)^d$, this introduces the
additional term $d\log(3D/\varepsilon)$ in the logarithmic confidence factor.

On this event, let
\[
\bx^\star\in\arg\min_{\bx\in\cK}\sum_{t=1}^T\ell_t(\bx),
\]
which exists because $\cK$ is compact and the losses are continuous. Under an adaptive adversary, $\bx^\star$ may be a random function of the learner's trajectory. This poses no difficulty here because the preceding union-bound event holds simultaneously for every point of the deterministic net $\mathcal N_\varepsilon$.
This is the uniformization step that is unavailable from a merely pointwise high-probability theorem without the union bound.
Choose $\by\in\mathcal N_\varepsilon$ such that $\|\by-\bx^\star\|\le\varepsilon$.
By $G$-Lipschitz continuity,
\[
\sum_{t=1}^T\ell_t(\by)
\le
\sum_{t=1}^T\ell_t(\bx^\star)+GT\varepsilon .
\]
Combining this approximation bound with the uniform fixed-comparator bound over
$\mathcal N_\varepsilon$ gives Eq.~\eqref{eq:uniform_cover}. The simplified
bound Eq.~\eqref{eq:uniform_cover_simplified} follows by substituting
$\varepsilon=D/T$: the term multiplied by the absolute constant $\Cabs$ absorbs
$\log 3$ into $\log(eT)$, while the explicit $4(d+1)GD$ term keeps the factor
$d\log(3T)$.
\end{proof}

\subsection{Detailed Proofs}

Our proof revisits the standard online gradient descent analysis of Hazan, Agarwal, and Kale~\citeyearpar{hazan2007logarithmic}, which achieves $O(\log T)$ regret for strongly convex losses.
This decomposition is the starting point for obtaining $O(\log T)$ regret with high probability in the two-point feedback setting.
\begin{lemma}
Suppose the loss function $\ell_t(\cdot)$ is $\mu$-strongly convex with $\mu>0$.
	Let the sequence $\{\bx_t\}$ be generated by Algorithm~\ref{alg:sgd}. 
For any $\bx \in (1-\xi)\cK$, it holds that
\begin{equation}\label{eq:regret}
\begin{aligned}
\sum_{t=1}^T \hat{\ell}_t(\bx_t) - \hat{\ell}_t(\bx)
\leq&
\sum_{t=1}^T \frac{\|\bx_t - \bx\|^2 - \|\bx_{t+1} - \bx\|^2}{2\eta_t} + \sum_{t=1}^T \frac{\eta_t}{2} \|\bg_t\|^2 
\\&
+ \sum_{t=1}^T \langle \nabla \hat{\ell}_t(\bx_t) - \bg_t, \bx_t - \bx \rangle - \frac{\mu}{2} \sum_{t=1}^T \|\bx_t - \bx\|^2.
\end{aligned}
\end{equation}
\end{lemma}
\begin{proof}
First, by nonexpansiveness of Euclidean projection onto the convex set $(1-\xi)\cK$,
\begin{align*}
	\|\bx_{t+1} - \bx\|^2 
	=& 
	\norm{ \Pi_{(1-\xi)\cK} (\bx_t - \eta_t \bg_t  ) -\bx}^2 
	\leq 
	\norm{ \bx_t - \eta_t \bg_t  -\bx }^2\\ 
	=& 
	\|\bx_t - \bx\|^2 - 2\eta_t \langle \bg_t, \bx_t - \bx \rangle + \eta_t^2 \|\bg_t\|^2.
\end{align*} 
Thus,
\begin{equation}\label{eq:reg}
	\langle \bg_t, \bx_t - \bx \rangle \leq \frac{\|\bx_t - \bx\|^2 - \|\bx_{t+1} - \bx\|^2}{2\eta_t} + \frac{\eta_t}{2} \|\bg_t\|^2.
\end{equation}
	
By the $\mu$-strong convexity of $\hat{\ell}_t$ from Lemma~\ref{lem:mu},
\begin{align*}
\hat{\ell}_t(\bx_t) - \hat{\ell}_t(\bx) \le \langle \nabla \hat{\ell}_t(\bx_t), \bx_t - \bx \rangle - \frac{\mu}{2} \|\bx_t - \bx\|^2.
\end{align*}
Thus,
\begin{align*}
\sum_{t=1}^T \hat{\ell}_t(\bx_t) - \hat{\ell}_t(\bx)
\le& 
\sum_{t=1}^T \langle \nabla \hat{\ell}_t(\bx_t), \bx_t - \bx \rangle - \frac{\mu}{2} \sum_{t=1}^T \|\bx_t - \bx\|^2
\\
=& 
\sum_{t=1}^T \langle \bg_t, \bx_t - \bx \rangle + \sum_{t=1}^T \langle \nabla \hat{\ell}_t(\bx_t) - \bg_t, \bx_t - \bx \rangle - \frac{\mu}{2} \sum_{t=1}^T \|\bx_t - \bx\|^2
\\
\stackrel{\eqref{eq:reg}}{\le}& 
\sum_{t=1}^T \frac{\|\bx_t - \bx\|^2 - \|\bx_{t+1} - \bx\|^2}{2\eta_t} + \sum_{t=1}^T \frac{\eta_t}{2} \|\bg_t\|^2 
\\&
+ \sum_{t=1}^T \langle \nabla \hat{\ell}_t(\bx_t) - \bg_t, \bx_t - \bx \rangle - \frac{\mu}{2} \sum_{t=1}^T \|\bx_t - \bx\|^2.
\end{align*}
\end{proof}

Taking expectation in Eq.~\eqref{eq:regret} gives
\begin{align*}
\EB\left[\sum_{t=1}^T \hat{\ell}_t(\bx_t) - \hat{\ell}_t(\bx)\right]
\leq 
\EB\left[
\sum_{t=1}^T \frac{\|\bx_t - \bx\|^2 - \|\bx_{t+1} - \bx\|^2}{2\eta_t}  
+ \sum_{t=1}^T \frac{\eta_t}{2}\|\bg_t\|^2
- \frac{\mu}{2} \sum_{t=1}^T \|\bx_t - \bx\|^2
\right].
\end{align*}
The original analysis of Agarwal, Dekel, and Xiao~\citeyearpar{agarwal2010optimal} combines the logarithmic-regret proof of Hazan, Agarwal, and Kale~\citeyearpar{hazan2007logarithmic} with a deterministic control on the two-point estimator, leading to an $O(d^2\log T)$ horizon-dependent expected term. Shamir's later geometric observation gives linear-in-$d$ second-moment control in expectation; our contribution is to obtain the corresponding weighted high-probability control under adaptivity.

To obtain a high-probability bound, we must control the extra term $\sum_{t=1}^T \langle \nabla \hat{\ell}_t(\bx_t) - \bg_t, \bx_t - \bx \rangle$, which vanishes in expectation.
A direct application of the Hoeffding--Azuma inequality incurs an additional $O(\sqrt{T})$ term in the analysis of Agarwal, Dekel, and Xiao~\citeyearpar{agarwal2010optimal}.

Next, the following lemma gives a high-probability upper bound on
$\sum_{t=1}^T \langle \nabla \hat{\ell}_t(\bx_t) - \bg_t, \bx_t - \bx \rangle$.
The proof uses a Freedman-style exponential supermartingale rather than a
direct Freedman tail bound with a random variance threshold.

\begin{lemma}\label{lem:Z_sum_freedman}
	Suppose the loss function $\ell_t(\bx)$ is $G$-Lipschitz continuous and the convex set
	$\mathcal{K}$ satisfies Eq.~\eqref{eq:DB}. Assume also Assumption~\ref{ass:nonanticipating}, and assume the algorithm maintains $\bx_t\in(1-\xi)\cK$ with $\alpha\le \xi r$ so all queried points lie in $\cK$. For any $\mu>0$, any fixed $\bx\in (1-\xi)\mathcal K$, and any $0<\delta<1$, with probability at least
	$1-\delta$, it holds that
	\begin{equation}\label{eq:Z_sum}
		\sum_{t=1}^T \langle \nabla \hat{\ell}_t(\bx_t) - \bg_t, \bx_t - \bx \rangle
		\le
		\frac{\mu}{4}\sum_{t=1}^{T}\|\bx_t - \bx\|^2
		+
		\frac{4 d G^2}{\mu}\log\frac{e}{\delta}
		+
		2(d+1)GD\log\frac{e}{\delta}.
	\end{equation}
\end{lemma}

\begin{proof}
	Let
	\[
	Z_t := \langle \nabla \hat{\ell}_t(\bx_t) - \bg_t, \bx_t - \bx \rangle.
	\]
	We use the filtration from Assumption~\ref{ass:nonanticipating}: conditionally on $\cF_{t-1}$, the loss $\ell_t$, the point $\bx_t$, and the comparator $\bx$ are fixed, while $\bu_t$ is the fresh randomness. Lemma~\ref{lem:g_prop} gives
	\[
	\mathbb{E}[\bg_t \mid \cF_{t-1}] = \nabla \hat{\ell}_t(\bx_t),
	\]
	we have
	\[
	\mathbb{E}[Z_t \mid \cF_{t-1}] = 0.
	\]
	Hence $(Z_t)_{t=1}^T$ is a martingale difference sequence.
	
	Next, we bound $|Z_t|$. By Cauchy--Schwarz,
	\begin{align*}
		|Z_t|
		&\le
		\|\nabla \hat{\ell}_t(\bx_t) - \bg_t\| \cdot \|\bx_t - \bx\|
		\\
		&\le
		\bigl(\|\nabla \hat{\ell}_t(\bx_t)\| + \|\bg_t\|\bigr)\|\bx_t - \bx\|.
	\end{align*}
	Because $\ell_t$ is $G$-Lipschitz continuous, Lemma~\ref{lem:mu} implies
	$\|\nabla \hat{\ell}_t(\bx_t)\|\le G$. Moreover,
	\begin{align*}
		\|\bg_t\|
		\stackrel{\eqref{eq:g}}{=}
		\frac{d}{2\alpha}
		\big|\ell_t(\bx_t+\alpha\bu_t)-\ell_t(\bx_t-\alpha\bu_t)\big|
		\|\bu_t\|
		\le
		\frac{d}{2\alpha}\cdot G \cdot 2\alpha \|\bu_t\|^2
		\le dG.
	\end{align*}
	Using Eq.~\eqref{eq:DB}, we have $\|\bx_t-\bx\|\le 2D$, and therefore
	\[
	|Z_t|\le 2(d+1)GD.
	\]
	Set
	\[
	b := 2(d+1)GD.
	\]
	
	Now we bound the conditional second moment. Since smoothing preserves
	Lipschitz continuity, Lemma~\ref{lem:mu} gives
	$\|\nabla\hat{\ell}_t(\bx_t)\|\le G$. Together with the conditional
	unbiasedness
	$\mathbb{E}[\bg_t\mid \cF_{t-1}] = \nabla \hat{\ell}_t(\bx_t)$,
	this yields the following conditional variance identity:
	\begin{align*}
		\mathbb{E}[Z_t^2\mid \cF_{t-1}]
		&=
		\mathbb{E}\Big[\big\langle \nabla \hat{\ell}_t(\bx_t)-\bg_t,\bx_t-\bx\big\rangle^2 \,\Big|\, \cF_{t-1}\Big]
		\\
		&=
		\mathbb{E}\Big[\langle \bg_t,\bx_t-\bx\rangle^2 \,\Big|\, \cF_{t-1}\Big]
		-
		\big\langle \nabla \hat{\ell}_t(\bx_t),\bx_t-\bx\big\rangle^2
		\\
		&\le
		\mathbb{E}\Big[\langle \bg_t,\bx_t-\bx\rangle^2 \,\Big|\, \cF_{t-1}\Big].
	\end{align*}
	Let $\by_t := \bx_t-\bx$, which is $\cF_{t-1}$-measurable. Then
	\begin{align*}
		\langle \bg_t,\by_t\rangle^2
		&=
		\left(
		\frac{d}{2\alpha}
		\big(\ell_t(\bx_t+\alpha\bu_t)-\ell_t(\bx_t-\alpha\bu_t)\big)
		\langle \bu_t,\by_t\rangle
		\right)^2
		\\
		&=
		\frac{d^2}{4\alpha^2}
		\big(\ell_t(\bx_t+\alpha\bu_t)-\ell_t(\bx_t-\alpha\bu_t)\big)^2
		\langle \bu_t,\by_t\rangle^2
		\\
		&\le
		\frac{d^2}{4\alpha^2}\cdot (2\alpha G)^2 \langle \bu_t,\by_t\rangle^2
		\\
		&=
		d^2 G^2 \langle \bu_t,\by_t\rangle^2,
	\end{align*}
	where we used the $G$-Lipschitz continuity of $\ell_t$. Thus the factor
	$d$ in the conditional variance comes from the spherical identity in the
	next display, rather than from the deterministic bound $\|\bg_t\|\le dG$.
	Taking conditional expectation and using the standard identity
	$\mathbb{E}[\langle \bu_t,\by_t\rangle^2\mid \cF_{t-1}]
	= \|\by_t\|^2/d$, we get
	\[
	\mathbb{E}[Z_t^2\mid \cF_{t-1}]
	\le
	d G^2 \|\bx_t-\bx\|^2.
	\]
	Define the predictable quadratic variation
	\[
	V_T := \sum_{t=1}^T \mathbb{E}[Z_t^2\mid \cF_{t-1}].
	\]
	Then
	\[
	V_T \le d G^2 \sum_{t=1}^T \|\bx_t-\bx\|^2.
	\]
	
	We now apply the Freedman-style exponential supermartingale lemma
	(Lemma~\ref{lem:freed_self_norm}) to $(Z_t)_{t=1}^T$: for every
	$\lambda\in[0,3/b)$, with probability at least $1-\delta$,
	\[
	\sum_{t=1}^T Z_t
	\le
	\frac{\lambda}{2(1-\lambda b/3)}V_T
	+
	\frac{1}{\lambda}\log\frac{1}{\delta}.
	\]
	Substituting the bound on $V_T$ yields
	\[
	\sum_{t=1}^T Z_t
	\le
	\frac{\lambda d G^2}{2(1-\lambda b/3)}
	\sum_{t=1}^T \|\bx_t-\bx\|^2
	+
	\frac{1}{\lambda}\log\frac{1}{\delta}.
	\]
	
	Choose
	\[
	\lambda := \frac{1}{b + 4dG^2/\mu}.
	\]
	Then $\lambda < 1/b < 3/b$, so the lemma applies. Since $\lambda b\le 1$,
	we have
	\[
	1-\lambda b/3 \ge \frac{2}{3},
	\]
	and the chosen value of $\lambda$ gives the direct estimate
	\[
	\lambda dG^2
	=
	\frac{dG^2}{b+4dG^2/\mu}
	\le
	\frac{\mu}{4}.
	\]
	Therefore
	\[
	\frac{\lambda d G^2}{2(1-\lambda b/3)}
	\le
	\frac{3}{4}\lambda dG^2
	\le
	\frac{3\mu}{16}
	\le
	\frac{\mu}{4},
	\]
	Hence
	\[
	\sum_{t=1}^T Z_t
	\le
	\frac{\mu}{4}\sum_{t=1}^T \|\bx_t-\bx\|^2
	+
		\left(b+\frac{4dG^2}{\mu}\right)\log\frac{e}{\delta}.
	\]
	Recalling $b=2(d+1)GD$, we conclude that with probability at least $1-\delta$,
	\[
	\sum_{t=1}^T \langle \nabla \hat{\ell}_t(\bx_t)-\bg_t,\bx_t-\bx\rangle
	\le
	\frac{\mu}{4}\sum_{t=1}^{T}\|\bx_t - \bx\|^2
	+
		\frac{4 d G^2}{\mu}\log\frac{e}{\delta}
		+
		2(d+1)GD\log\frac{e}{\delta}.
	\]
	This completes the proof.
\end{proof}

The bound in Eq.~\eqref{eq:Z_sum} is the step that keeps the high-probability
regret logarithmic. Its leading term is
$\frac{\mu}{4}\sum_{t=1}^T\norm{\bx_t - \bx}^2$, which can be absorbed by the
curvature term in Eq.~\eqref{eq:regret}:
\begin{align*}
\sum_{t=1}^T \frac{\|\bx_t - \bx\|^2 - \|\bx_{t+1} - \bx\|^2}{2\eta_t} 
-\frac{\mu}{2} \sum_{t=1}^T \|\bx_t - \bx\|^2.
\end{align*} 
With the step size $\eta_t = 2/(\mu t)$, this cancellation yields the following
bound with no additional $\sqrt T$ term.

\begin{lemma}
Suppose the loss function $\ell_t(\cdot)$ is $\mu$-strongly convex with $\mu>0$ and $G$-Lipschitz continuous. Assume Assumption~\ref{ass:nonanticipating} and $\alpha\le \xi r$.
Let the sequence $\{\bx_t\}$ be generated by Algorithm~\ref{alg:sgd}. 
Set $\eta_t = 2/(\mu t)$. For any fixed $\bx \in (1-\xi)\cK$ and any $0<\delta<1$, with probability at least $1-\delta$,
\begin{equation}\label{eq:reg_2}
\sum_{t=1}^T \hat{\ell}_t(\bx_t) - \hat{\ell}_t(\bx)
\leq
\sum_{t=1}^T \frac{\norm{\bg_t}^2}{\mu t} 
+ \frac{4 d G^2}{\mu}\log\frac{e}{\delta}
	+
	2(d+1)GD\log\frac{e}{\delta}.
\end{equation}
\end{lemma}
\begin{proof}
We have
\begin{align*}
\sum_{t=1}^T \hat{\ell}_t(\bx_t) - \hat{\ell}_t(\bx)
\stackrel{\eqref{eq:regret}}{\leq}&
\sum_{t=1}^T \frac{\|\bx_t - \bx\|^2 - \|\bx_{t+1} - \bx\|^2}{2\eta_t} + \sum_{t=1}^T \frac{\eta_t}{2} \|\bg_t\|^2 
\\&
+ \sum_{t=1}^T \langle \nabla \hat{\ell}_t(\bx_t) - \bg_t, \bx_t - \bx \rangle - \frac{\mu}{2} \sum_{t=1}^T \|\bx_t - \bx\|^2
\\
\stackrel{\eqref{eq:Z_sum}}{\leq}&
\sum_{t=1}^T \frac{\|\bx_t - \bx\|^2 - \|\bx_{t+1} - \bx\|^2}{2\eta_t} + \sum_{t=1}^T \frac{\eta_t}{2} \|\bg_t\|^2
\\&
- \frac{\mu}{4}\sum_{t=1}^{T}\|\bx_t - \bx\|^2  
+ \frac{4 d G^2}{\mu}\log\frac{e}{\delta}
	+
	2(d+1)GD\log\frac{e}{\delta}
\\
=&
\sum_{t=1}^T  \frac{\mu t\left(\|\bx_t - \bx\|^2 - \|\bx_{t+1} - \bx\|^2\right)}{4} 
- \frac{\mu}{4}\sum_{t=1}^{T}\|\bx_t - \bx\|^2
\\
&+ \sum_{t=1}^T \frac{\norm{\bg_t}^2}{\mu t} 
+ \frac{4 d G^2}{\mu}\log\frac{e}{\delta}
	+
	2(d+1)GD\log\frac{e}{\delta},
\end{align*}
where the last equality is because of the step size $\eta_t = \frac{2}{\mu t}$.

The distance terms telescope:
\begin{align*}
&\sum_{t=1}^T  \frac{\mu t\left(\|\bx_t - \bx\|^2 - \|\bx_{t+1} - \bx\|^2\right)}{4} 
- \frac{\mu}{4}\sum_{t=1}^{T}\|\bx_t - \bx\|^2
\\
=&
\mu \sum_{t=1}^T \frac{(t-1)\|\bx_t - \bx\|^2 - t  \|\bx_{t+1} - \bx\|^2}{4}
\\
=&
\mu \cdot \frac{  - T \|\bx_{T+1} - \bx\|^2 }{4}
\\
\leq& 
0.
\end{align*}
Combining the preceding bounds gives the result.
\end{proof}

It remains to show that $\sum_{t=1}^T \norm{\bg_t}^2/(\mu t)$ is upper bounded
by $O(d\log T)$ with high probability. This step gives
Algorithm~\ref{alg:sgd} linear dimension dependence in the leading
fixed-comparator term. We extend the geometric technique of
Shamir~\citeyearpar{shamir2017optimal}, originally used for expectation
bounds, to a high-probability statement.

\begin{lemma}
	Suppose the losses are $G$-Lipschitz on $\cK$, Assumption~\ref{ass:nonanticipating} holds, and Algorithm~\ref{alg:sgd} maintains $\bx_t\in(1-\xi)\cK$ with $\alpha\le \xi r$. Let $\bg_t$ be defined as in Eq.~\eqref{eq:g}. Given $\mu>0$ and $0<\delta<1$, with probability at least $1-\delta$, it holds that
	\begin{equation}\label{eq:G_sum}
		\sum_{t=1}^T \frac{\norm{\bg_t}^2}{\mu t} 
		\leq 
		\Cabs \cdot \frac{d G^2}{\mu}\cdot \left( \log(eT) + \log \frac{e}{\delta} \right),
	\end{equation}
	where $\Cabs$ is an absolute constant.
\end{lemma}
\begin{proof}
{
If $d=1$, then the Lipschitz condition gives $\|\bg_t\|\le G$ deterministically, and hence
\[
\sum_{t=1}^T\frac{\|\bg_t\|^2}{\mu t}
\le
\frac{G^2}{\mu}\sum_{t=1}^T\frac{1}{t}
\le
\frac{G^2}{\mu}\log(eT),
\]
which is covered by the stated bound. We therefore assume $d\ge2$ below.
For each $t$, define the conditional means
\[
m_t^+
:=
\EB[\ell_t(\bx_t+\alpha\bu_t)\mid\cF_{t-1}],
\qquad
m_t^-
:=
\EB[\ell_t(\bx_t-\alpha\bu_t)\mid\cF_{t-1}],
\]
and the centered spherical fluctuations
\[
X_t^+
:=
\ell_t(\bx_t+\alpha\bu_t)-m_t^+,
\qquad
X_t^-
:=
\ell_t(\bx_t-\alpha\bu_t)-m_t^-.
\]
Conditionally on $\cF_{t-1}$, the loss $\ell_t$ and the point $\bx_t$ are
fixed, while $\bu_t$ is uniform on the sphere. The maps
$\bu\mapsto\ell_t(\bx_t+\alpha\bu)$ and
$\bu\mapsto\ell_t(\bx_t-\alpha\bu)$ are both $(G\alpha)$-Lipschitz on
$\SB^{d-1}$ with respect to Euclidean distance. Applying Lemma~\ref{lem:sum_X} with failure probability
$\delta/2$ to the $+$ fluctuations gives
\[
\sum_{t=1}^T \frac{(X_t^+)^2}{t}
\le
\Cabs \cdot \frac{(G\alpha)^2}{d}
\left( \log(eT) + \log \frac{2e}{\delta} \right).
\]
Applying the same lemma with failure probability $\delta/2$ to the $-$
fluctuations gives
\[
\sum_{t=1}^T \frac{(X_t^-)^2}{t}
\le
\Cabs \cdot \frac{(G\alpha)^2}{d}
\left( \log(eT) + \log \frac{2e}{\delta} \right).
\]
By a union bound, the two displays hold simultaneously with probability at
least $1-\delta$. On this event, Eq.~\eqref{eq:g_norm} gives the desired
linear-in-$d$ control: the factor $d^2/\alpha^2$ in the estimator norm is
multiplied by the spherical concentration scale $(G\alpha)^2/d$. Absorbing the
constant $\log 2$ into $\Cabs$, we obtain}
\begin{align*}
		\sum_{t=1}^T \frac{\norm{\bg_t}^2}{\mu t} 
		\stackrel{\eqref{eq:g_norm}}{\leq}&
		\frac{3d^2}{4\alpha^2 \mu} \sum_{t=1}^T \frac{(X_t^+)^2+(X_t^-)^2}{t} 
		\\
		\leq& 
		\frac{3d^2}{4\alpha^2 \mu} \cdot \left( 2\Cabs \cdot \frac{(G\alpha)^2}{d} \cdot \left( \log(eT) + \log \frac{2e}{\delta} \right)\right)
		\\
		=&
		\Cabs \cdot \frac{d G^2}{\mu}\cdot \left( \log(eT) + \log \frac{e}{\delta} \right).
\end{align*}
\end{proof}

\begin{lemma}
\label{lem:reg_3}
Suppose the loss function $\ell_t(\cdot)$ is $\mu$-strongly convex with $\mu>0$ and $G$-Lipschitz continuous. Assume Assumption~\ref{ass:nonanticipating} and $\alpha\le \xi r$.
Let the sequence $\{\bx_t\}$ be generated by Algorithm~\ref{alg:sgd}. 
For any fixed $\bx \in (1-\xi)\cK$ and any $0<\delta<1$, with probability at least $1-\delta$,
\begin{equation}\label{eq:reg_3}
\sum_{t=1}^T \hat{\ell}_t(\bx_t) - \hat{\ell}_t(\bx)
\leq
\Cabs \cdot \frac{d G^2}{\mu}\cdot \left( \log(eT) + \log \frac{e}{\delta} \right)
+ 
4(d+1)GD\log\frac{e}{\delta}.
\end{equation}
\end{lemma}
\begin{proof}
Apply Eq.~\eqref{eq:reg_2} with failure probability $\delta/2$ and Eq.~\eqref{eq:G_sum} with failure probability $\delta/2$. By the union bound, both events hold with probability at least $1-\delta$. On this event,
\begin{align*}
\sum_{t=1}^T \hat{\ell}_t(\bx_t) - \hat{\ell}_t(\bx)
\leq&
\sum_{t=1}^T \frac{\norm{\bg_t}^2}{\mu t} 
+ \frac{4 d G^2}{\mu}\log\frac{2e}{\delta}
	+
	2(d+1)GD\log\frac{2e}{\delta}
\\
\leq&
\Cabs \cdot \frac{d G^2}{\mu}\cdot \left( \log(eT) + \log \frac{2e}{\delta} \right)
+ \frac{4 d G^2}{\mu}\log\frac{2e}{\delta}
	+
	2(d+1)GD\log\frac{2e}{\delta}\\
\leq&
\Cabs \cdot \frac{d G^2}{\mu}\cdot \left( \log(eT) + \log \frac{e}{\delta} \right) 
+
4(d+1)GD\log\frac{e}{\delta}.
\end{align*}
Here we used $\log(2e/\delta)\le 2\log(e/\delta)$ for $\delta\in(0,1)$ and enlarged the absolute constant $\Cabs$.
\end{proof}

Combining the preceding lemmas proves Theorem~\ref{thm:main}.
\begin{proof}[Proof of Theorem~\ref{thm:main}]
Since $0\in\cK$, the initialization gives $\bx_1=0\in(1-\xi)\cK$. The projection step in Algorithm~\ref{alg:sgd} therefore maintains $\bx_t\in(1-\xi)\cK$ for every round. The parameter choice also gives $\alpha=\xi r/2$, so Lemma~\ref{lem:query_feasible} applies with the interior-margin condition.
Combining Lemma~\ref{lem:diff} and Lemma~\ref{lem:reg_3} gives
	\begin{align*}
		&\sum_{t=1}^{T}\frac{ \ell_t(\bx_t + \alpha \bu_t) + \ell_t(\bx_t - \alpha \bu_t)}{2} - \sum_{t=1}^T \ell_t(\bx)
		\\
		\stackrel{\eqref{eq:diff}}{\leq}&
		\sum_{t=1}^T \hat{\ell}_t(\bx_t) - \sum_{t=1}^T \hat{\ell}_t((1-\xi)\bx) + 3TG\alpha + TGD\xi
		\\
		\stackrel{\eqref{eq:reg_3}}{\leq}&
		\Cabs \cdot \frac{d G^2}{\mu}\cdot \left( \log(eT) + \log \frac{e}{\delta} \right) 
	+
	4(d+1)GD\log\frac{e}{\delta}
		+ 3TG\alpha + TGD\xi\\
		=&
		\Cabs \cdot \frac{d G^2}{\mu}\cdot \left( \log(eT) + \log \frac{e}{\delta} \right) 
	+
	4(d+1)GD\log\frac{e}{\delta}
		+
		G\log(eT)\left(3 + \frac{2D}{r}\right),
	\end{align*}
	where the second inequality is also because $(1-\xi)\bx \in (1-\xi)\cK$ if $\bx \in \cK$, and the last equality uses $\alpha = \frac{\log(eT)}{T}$ and $\xi = 2\alpha / r$. The condition $2\alpha\le r$ gives $\xi\le 1$ and $\alpha=\xi r/2$, so Lemma~\ref{lem:query_feasible} ensures that the smoothing points and the two queried points remain in $\operatorname{int}(\mathcal K)$.
\end{proof}

\section{Conclusion}

We proved a logarithmic high-probability fixed-comparator regret bound for
strongly convex online convex optimization with two-point bandit feedback. This
is the high-probability counterpart of the comparator-wise expected theorem of
Agarwal, Dekel, and Xiao~\citeyearpar{agarwal2010optimal}, whose
horizon-dependent term scales as $O(d^2\log T)$.
Our bound is
\[
O\left(\frac{dG^2}{\mu}\left(\log T+\log(1/\delta)\right)+dGD\log(1/\delta)+G\log T\left(1+\frac{D}{r}\right)\right),
\]
whose leading fixed-comparator term has logarithmic horizon dependence and
linear dimension dependence. The expected theorem of Agarwal, Dekel, and
Xiao can be minimized to obtain a pseudo-regret bound, whereas our covering
corollary gives a high-probability bound against the realized hindsight
comparator. The latter requires a uniformization step and therefore pays the
covering-number factor. At the fixed-comparator level, this supplies the
missing logarithmic high-probability ingredient; after standard covering, it
yields the corresponding realized full-comparator guarantee.

Technically, the proof departs from the standard reduction-based analysis.
It develops a direct high-confidence framework for multi-point feedback and
combines martingale concentration with geometric control of the two-point
estimator.
At the fixed-comparator level, this reduces the dimension dependence of the
horizon-dependent term from $O(d^2)$ to $O(d)$. This fixed-comparator bound is
the technical core that enables the realized full-comparator logarithmic
high-probability corollary after uniformization.

\newpage

\appendix

\section{Useful Lemmas}

\begin{lemma}[Corollary 2.6 of Ledoux~\citeyearpar{ledoux2001concentration}]
\label{lem:s}
Let $\bu \sim \SB^{d-1}$ be uniformly drawn from the unit sphere in $\mathbb R^d$. Then
\begin{equation}\label{eq:s}
	\Pr\left(| h(\bu) - \EB [h(\bu)] | \ge \tau \right) \leq 2 \exp\left(-\frac{c_1 d \tau^2}{L^2}\right),
\end{equation} 
where $c_1$ is an absolute constant and $h$ is $L$-Lipschitz on $\SB^{d-1}$ with respect to Euclidean distance.
\end{lemma}

\begin{lemma}[Freedman-style exponential supermartingale]\label{lem:freed_self_norm}
	Let $(Z_t,\mathcal F_t)$ be a martingale difference sequence such that
	$|Z_t| \le b$ almost surely for all $t$, where $b>0$ is a constant. Define
	\[
	V_t := \sum_{s=1}^t \mathbb{E}[Z_s^2\mid \mathcal F_{s-1}].
	\]
	Then for every $\lambda\in[0,3/b)$, the process
	\[
	M_t(\lambda)
	:=
	\exp\!\left(
	\lambda\sum_{s=1}^t Z_s
	-
	\frac{\lambda^2}{2(1-\lambda b/3)}V_t
	\right)
	\]
	is a nonnegative supermartingale. Consequently, for every $\delta\in(0,1)$ and every $\lambda\in(0,3/b)$,
	with probability at least $1-\delta$,
	\[
	\sum_{t=1}^T Z_t
	\le
	\frac{\lambda}{2(1-\lambda b/3)}V_T
	+
	\frac{1}{\lambda}\log\frac{1}{\delta}.
	\]
\end{lemma}
\begin{proof}
{
	Fix $\lambda\in[0,3/b)$ and set
	\[
	a_\lambda:=\frac{\lambda^2}{2(1-\lambda b/3)}.
	\]
	For every $|x|\le b$,
	\[
	e^{\lambda x}
	=
	1+\lambda x+\sum_{k=2}^\infty \frac{\lambda^k x^k}{k!}.
	\]
	Since $x^k\le |x|^k\le x^2 b^{k-2}$ for all $k\ge 2$, we have
	\[
	e^{\lambda x}
	\le
	1+\lambda x+\frac{x^2}{b^2}\sum_{k=2}^\infty \frac{(\lambda b)^k}{k!}
	=
	1+\lambda x+\frac{\phi(\lambda b)}{b^2}x^2.
	\]
	For $0\le u<3$,
	\[
	\phi(u)=\sum_{k=2}^\infty \frac{u^k}{k!}
	\le
	\sum_{k=2}^\infty \frac{u^k}{2\cdot 3^{k-2}}
	=
	\frac{u^2}{2(1-u/3)}.
	\]
	Thus $\phi(\lambda b)/b^2\le a_\lambda$, and therefore
	\[
	e^{\lambda x}\le 1+\lambda x+a_\lambda x^2.
	\]
	Applying this with $x=Z_t$ and taking conditional expectation given $\mathcal F_{t-1}$ gives
	\[
	\mathbb E[e^{\lambda Z_t}\mid \mathcal F_{t-1}]
	\le
	1+a_\lambda\mathbb E[Z_t^2\mid \mathcal F_{t-1}],
	\]
	where we used $\mathbb E[Z_t\mid \mathcal F_{t-1}]=0$. Hence, using $1+u\le e^u$,
	\[
	\mathbb E[e^{\lambda Z_t}\mid \mathcal F_{t-1}]
	\le
	\exp\!\left(
	a_\lambda\mathbb E[Z_t^2\mid \mathcal F_{t-1}]
	\right).
	\]
	\[
	M_t(\lambda)
	=
	M_{t-1}(\lambda)
	\exp\!\left(
	\lambda Z_t
	-
	a_\lambda\mathbb E[Z_t^2\mid\mathcal F_{t-1}]
	\right),
	\]
	and the preceding display implies
	\[
	\mathbb E[M_t(\lambda)\mid \mathcal F_{t-1}]
	\le M_{t-1}(\lambda).
	\]
	Thus $(M_t(\lambda))_{t\ge0}$ is a nonnegative supermartingale with $M_0(\lambda)=1$.
	
	Applying Markov's inequality, for any $\delta\in(0,1)$,
	\begin{align*}
		\Pr\!\left(
		\lambda \sum_{t=1}^T Z_t
		-
		a_\lambda V_T
		\ge
		\log\frac{1}{\delta}
		\right)
		&=
		\Pr\!\left(M_T(\lambda)\ge \frac{1}{\delta}\right)
		\\
		&\le
		\delta\,\mathbb{E}[M_T(\lambda)]
	\le \delta.
	\end{align*}
	Equivalently, with probability at least $1-\delta$,
	\[
	\sum_{t=1}^T Z_t
	\le
	\frac{a_\lambda}{\lambda}V_T
	+
	\frac{1}{\lambda}\log\frac{1}{\delta}.
	\]
	Since $a_\lambda/\lambda=\lambda/(2(1-\lambda b/3))$, the claimed bound follows.
}
\end{proof}

Next, we define the Orlicz norm. 
Then we will introduce definitions and properties of subgaussian and subexponential distribution.
\begin{definition}
Let $X$ be a random variable, we can define its Orlicz norms $\|X\|_{\psi_2}$ and $\|X\|_{\psi_1}$ as follows:
\[
\|X\|_{\psi_2} = \inf \left\{ \tau > 0 : \mathbb{E} \exp \left( \frac{X^2}{\tau^2} \right) \le 2 \right\},
\]
and
\begin{equation}\label{eq:psi_1}
\|X\|_{\psi_1} = \inf \left\{ \tau > 0 : \mathbb{E} \exp \left( \frac{|X|}{\tau} \right) \le 2 \right\}.	
\end{equation}
\end{definition}

\begin{lemma}[Proposition 2.6.1 of Vershynin~\citeyearpar{Vershynin2026}]
	\label{lem:subgau_1}
Let $X$ be a random \emph{subgaussian} variable. The following properties are equivalent, with the parameters $K_i > 0$ differing by at most an absolute constant factor.
\begin{enumerate}
	\item[(i)] \textit{(Tails)} There exists $K_1 > 0$ such that
	\begin{equation}\label{eq:subgau}
		\Pr \{ |X| \ge \tau \} \le 2 \exp(-\tau^2/K_1^2) \quad \text{for all } \tau \ge 0.
	\end{equation}
	
	\item[(ii)] \textit{(Moments)} There exists $K_2 > 0$ such that
	\[
	\|X\|_{L^p} = (\mathbb{E}|X|^p)^{1/p} \le K_2 \sqrt{p} \quad \text{for all } p \ge 1.
	\]
	
	\item[(iii)] \textit{(MGF of $X^2$)} There exists $K_3 > 0$ such that
	\[
	\mathbb{E} \exp(X^2 / K_3^2) \le 2.
	\]
\end{enumerate}
Moreover, if $\mathbb{E}X = 0$ then properties $(i)$--$(iii)$ are equivalent to the following one:

\begin{enumerate}
	\item[(iv)] \textit{(MGF)} There exists $K_4 > 0$ such that
	\[
	\mathbb{E} \exp(\lambda X) \le \exp(K_4^2 \lambda^2) \quad \text{for all } \lambda \in \mathbb{R}.
	\]
\end{enumerate}
\end{lemma}

The preceding properties can also be represented by Orlicz norms as follows.

\begin{lemma}[Proposition 2.6.6 of Vershynin~\citeyearpar{Vershynin2026}]
	\label{lem:subgau_2}
Every subgaussian random variable $X$ satisfies the following bounds.

\begin{enumerate}
	\item[(i)] \textit{(Tails)} $\Pr \{ |X| \ge \tau \} \le 2 \exp \left( -c\tau^2 / \|X\|_{\psi_2}^2 \right)$ for all $\tau \ge 0$. 
	
	\item[(ii)] \textit{(Moments)} $\|X\|_{L^p} \le C \|X\|_{\psi_2} \sqrt{p}$ for all $p \ge 1$.
	
	\item[(iii)] \textit{(MGF of $X^2$)} $\mathbb{E} \exp \left( X^2 / \|X\|_{\psi_2}^2 \right) \le 2$.  
	
	\item[(iv)] \textit{(MGF)} If $\mathbb{E}X = 0$ then $\mathbb{E} \exp(\lambda X) \le \exp \left( C \lambda^2 \|X\|_{\psi_2}^2 \right)$ for all $\lambda \in \mathbb{R}$.
\end{enumerate}
Here $C, c > 0$ are absolute constants. Moreover, up to absolute constant factors, $\|X\|_{\psi_2}$ is the smallest possible number that makes each of these statements valid.
\end{lemma}

\begin{lemma}[Proposition 2.8.1 of Vershynin~\citeyearpar{Vershynin2026}]
\label{lem:sub_exp}
Let $X$ be a random \emph{subexponential} variable. The following properties are equivalent, with the parameters $K_i > 0$ differing by at most an absolute constant factor.

\begin{enumerate}
	\item[(i)] \textit{(Tails)} There exists $K_1 > 0$ such that
	\[
	\mathbb{P} \{ |X| \ge t \} \le 2 \exp(-t/K_1) \quad \text{for all } t \ge 0.
	\]
	
	\item[(ii)] \textit{(Moments)} There exists $K_2 > 0$ such that
	\[
	\|X\|_{L^p} = (\mathbb{E}|X|^p)^{1/p} \le K_2 p \quad \text{for all } p \ge 1.
	\]
	
	\item[(iii)] \textit{(MGF of $|X|$)} There exists $K_3 > 0$ such that
	\[
	\mathbb{E} \exp(|X| / K_3) \le 2.
	\]
\end{enumerate}

Moreover, if $\mathbb{E}X = 0$ then properties $(i)$--$(iii)$ are equivalent to the following one:

\begin{enumerate}
	\item[(iv)] \textit{(MGF)} There exists $K_4 > 0$ such that
	\[
	\mathbb{E} \exp(\lambda X) \le \exp(K_4^2 \lambda^2) \quad \text{for all } \lambda \text{ such that } |\lambda| \le \frac{1}{K_4}.
	\]
\end{enumerate}
\end{lemma}

\begin{lemma}[Lemma 2.8.5 of Vershynin~\citeyearpar{Vershynin2026}]
	\label{lem:sss}
	The random variable $X$ is subgaussian if and
	only if $X^2$ is subexponential, and
	\begin{align*}
		\norm{X^2}_{\psi_1} = \norm{X}^2_{\psi_2}.
	\end{align*}
\end{lemma}

\section{Proofs Related to \texorpdfstring{$\norm{g_t}^2$}{the squared norm of g\_t}}

\begin{lemma}\label{lem:spherical_orlicz}
	Letting $\bu \sim \SB^{d-1}$ and function $h(\cdot)$ be $L$-Lipschitz with respect to Euclidean distance on $\SB^{d-1}$, then random variable $X:= h(\bu) - \mathbb{E}[h(\bu)]$ is a subgaussian random variable which satisfies that
	\begin{equation*}
		\norm{X}_{\psi_2} \leq C_1 \frac{L}{\sqrt{d}}.
	\end{equation*}
	Furthermore, the random variable $Y := X^2$ is   a sub-exponential random variable with 
	\begin{equation}\label{eq:Y}
		\norm{Y}_{\psi_1} = \norm{X}^2_{\psi_2} \leq \frac{C_1^2 L^2}{d}.
	\end{equation}
	For the centered random variable $V = Y - \mathbb{E}[Y]$, it holds that
	\begin{equation}\label{eq:V}
		\norm{V}_{\psi_1} \leq C_2\|Y\|_{\psi_1}.
	\end{equation}
	Here $C_1, C_2 >0$ are absolute constants. 
\end{lemma}
\begin{proof}
	First, Lemma~\ref{lem:s} and Eq.~\eqref{eq:subgau} imply that $X$ is a subgaussian random variable.
	Furthermore, Eq.~\eqref{eq:s}, item (i), and item (iii) in Lemma~\ref{lem:subgau_2} imply that
	\begin{equation*}
		\norm{X}_{\psi_2} \leq C_1 \frac{L}{\sqrt{d}}.
	\end{equation*}
	
	By Lemma~\ref{lem:sss}, we can obtain the result about random variable $Y$.
	
	Finally, we have
	\begin{align*}
		\norm{V}_{\psi_1} = \norm{ Y - \EB\left[Y\right] }_{\psi_1} \leq C_2\|Y\|_{\psi_1}.
	\end{align*}
\end{proof}

\begin{lemma}
The approximate gradient $\bg_t$ defined in Eq.~\eqref{eq:g} has its norm bounded as
\begin{equation}\label{eq:g_norm}
\begin{aligned}
\| \bg_t \|^2
\leq\;&
\frac{3d^2}{4\alpha^2}
\Big(
\left( \ell_t(\bx_t + \alpha \bu_t)
- \EB [\ell_t(\bx_t + \alpha \bu_t)\mid\cF_{t-1}] \right)^2
\\
&\qquad\qquad+
\left( \EB [\ell_t(\bx_t - \alpha \bu_t)\mid\cF_{t-1}]
- \ell_t(\bx_t - \alpha \bu_t) \right)^2
\Big).
\end{aligned}
\end{equation}
\end{lemma}
\begin{proof}
Let
\[
m_t^+:=\EB [\ell_t(\bx_t + \alpha \bu_t)\mid\cF_{t-1}],
\qquad
m_t^-:=\EB [\ell_t(\bx_t - \alpha \bu_t)\mid\cF_{t-1}].
\]
Under Assumption~\ref{ass:nonanticipating}, $\ell_t$ and $\bx_t$ are fixed
conditionally on $\cF_{t-1}$, and $\bu_t$ is uniform on the sphere. Hence the
change of variables $\bu_t\mapsto-\bu_t$ gives $m_t^+=m_t^-$. Using the
definition of $\bg_t$, we have 
\begin{align*}
	\| \bg_t \|^2 
	&\stackrel{\eqref{eq:g}}{=} \frac{d^2}{4\alpha^2} \| \bu_t \|^2 \left( \ell_t(\bx_t + \alpha \bu_t) - \ell_t(\bx_t - \alpha \bu_t) \right)^2 \\
	&\le \frac{3d^2}{4\alpha^2} \left[ \left( \ell_t(\bx_t + \alpha \bu_t) - m_t^+ \right)^2 \right. \\
	&\quad + \left( m_t^+ -  m_t^-  \right)^2 + \left. \left( m_t^- - \ell_t(\bx_t - \alpha \bu_t) \right)^2 \right] \\
	&= \frac{3d^2}{4\alpha^2} \left( \left( \ell_t(\bx_t + \alpha \bu_t) - m_t^+ \right)^2 + \left( m_t^- - \ell_t(\bx_t - \alpha \bu_t) \right)^2\right),
\end{align*}
which is the claimed bound.
\end{proof}

\begin{lemma}[Weighted conditional sub-exponential summation]
\label{lem:weighted_subexp}
Let $(W_t,\cF_t)_{t=1}^T$ be a martingale difference sequence. Suppose there
exist constants $K>0$ and absolute constants $c_0,C_0>0$ such that, almost
surely, for every $t$ and every $|\lambda|\le c_0/K$,
\[
\EB[\exp(\lambda W_t)\mid\cF_{t-1}]
\le
\exp(C_0\lambda^2K^2).
\]
Then, for every $0<\delta<1$, with probability at least $1-\delta$,
\[
\sum_{t=1}^T \frac{W_t}{t}
\le
C K\log\frac{e}{\delta},
\]
where $C>0$ is an absolute constant.
\end{lemma}
\begin{proof}
Let $S_T:=\sum_{t=1}^T W_t/t$. Iterating the conditional moment-generating
function bound gives, for all $|\lambda|\le c_0/K$,
\[
\EB\exp(\lambda S_T)
\le
\exp\left(C_0\lambda^2K^2\sum_{t=1}^T\frac{1}{t^2}\right)
\le
\exp(C_1\lambda^2K^2)
\]
for an absolute constant $C_1>0$. In the conditional step at time $t$, the
assumed MGF bound is applied to $\lambda W_t/t$, which is valid because
$t\ge 1$ implies $|\lambda|/t\le |\lambda|\le c_0/K$. Chernoff's bound
therefore gives
\[
\Pr(S_T\ge s)
\le
\inf_{0\le\lambda\le c_0/K}
\exp\left(-\lambda s+C_1\lambda^2K^2\right)
\le
\exp\left(-c\min\left\{\frac{s^2}{K^2},\frac{s}{K}\right\}\right)
\]
with an absolute constant $c>0$. Taking
$s=C K\log(e/\delta)$ with $C$ sufficiently large proves the claim.
\end{proof}

\begin{lemma}
\label{lem:sum_X}
{
Assume $d\ge2$. Suppose the losses are $G$-Lipschitz on $\cK$, Assumption~\ref{ass:nonanticipating} holds, and $\bx_t\in(1-\xi)\cK$ with $\alpha\le \xi r$ for every $t$.
For either fixed sign $\sigma\in\{-1,1\}$, define
\[
X_t =
\ell_t(\bx_t+\sigma\alpha\bu_t)
-
\EB\!\left[\ell_t(\bx_t+\sigma\alpha\bu_t)\mid\cF_{t-1}\right].
\]
Given $0<\delta<1$, with probability at least $1-\delta$,}
\begin{equation*}
\sum_{k=1}^{T} \frac{X_k^2}{k} 
\leq
\Cabs \cdot \frac{(G\alpha)^2}{d} \cdot \left( \log(eT) + \log \frac{e}{\delta} \right),
\end{equation*} 
where $\Cabs$ is an absolute constant.
\end{lemma}
\begin{proof}
{
Let
\[
Y_t:=X_t^2,\qquad
W_t:=Y_t-\EB[Y_t\mid\cF_{t-1}].
\]
Then $(W_t,\cF_t)$ is a martingale difference sequence. Conditionally on
$\cF_{t-1}$, the map $\bu\mapsto \ell_t(\bx_t+\sigma\alpha\bu)$ is
$(G\alpha)$-Lipschitz on the unit sphere with respect to Euclidean distance. After freezing $\cF_{t-1}$, Lemma~\ref{lem:s}
is exactly the standard sphere concentration inequality for this deterministic
Lipschitz function. Equivalently, the argument of Lemma~\ref{lem:spherical_orlicz}
applies conditionally on $\cF_{t-1}$ and implies
\[
\|X_t\|_{\psi_2\mid\cF_{t-1}}\le C\frac{G\alpha}{\sqrt d}.
\]
By the conditional convention stated after Assumption~\ref{ass:nonanticipating}, the
constants are absolute and uniform over the frozen history, because the only
quantity used in the concentration bound is the Lipschitz constant $G\alpha$.
Consequently, with
\[
K_0:=C\frac{(G\alpha)^2}{d},
\]
we have, almost surely,
\[
\EB[Y_t\mid\cF_{t-1}]\le K_0,
\qquad
\|W_t\|_{\psi_1\mid\cF_{t-1}}\le K_0.
\]
Therefore, by the conditional sub-exponential MGF bound, there exist absolute
constants $c,C>0$ such that for all $|\lambda|\le c/K_0$,
\[
\EB\!\left[\exp(\lambda W_t)\mid\cF_{t-1}\right]
\le \exp(C\lambda^2 K_0^2).
\]
Applying Lemma~\ref{lem:weighted_subexp} gives, with probability at least
$1-\delta$,
\[
\sum_{k=1}^T\frac{W_k}{k}
\le
C K_0\log\frac{e}{\delta}.
\]
On the same event,
\begin{align*}
\sum_{k=1}^T\frac{X_k^2}{k}
&=
\sum_{k=1}^T\frac{W_k}{k}
+
\sum_{k=1}^T\frac{\EB[Y_k\mid\cF_{k-1}]}{k}
\\
&\le
C K_0\log\frac{e}{\delta}
+
K_0\sum_{k=1}^T\frac{1}{k}
\\
&\le
\Cabs\frac{(G\alpha)^2}{d}
\left(\log(eT)+\log\frac{e}{\delta}\right).
\end{align*}
}
\end{proof}
\vskip 0.2in
\bibliography{ref.bib,ref_recent.bib}
\bibliographystyle{apalike2}

\end{document}